\documentclass[11pt]{article}

\usepackage[margin=1in]{geometry}
\usepackage{amsmath,amssymb,amsthm}
\usepackage{graphicx}
\usepackage{booktabs}
\usepackage{algorithm}
\usepackage{algpseudocode}
\usepackage{hyperref}
\usepackage{natbib}  
\usepackage{caption} 
\usepackage{booktabs}
\newtheorem{theorem}{Theorem}
\newtheorem{lemma}[theorem]{Lemma}

\usepackage{booktabs}
\usepackage{pgfplots}
\newtheorem{remark}{Remark}
\newtheorem{definition}{Definition}
\newtheorem{corollary}[theorem]{Corollary}
\usepackage{pgfplots}
\pgfplotsset{compat=1.18}
\usepackage{newfloat}
\usepackage{listings}
\title{A Geometric Theory of Robust Fairness Audits}

\author{
 Binita Maity\\
Indian Institute of Technology, Gandhinagar\\
\texttt{binitamaity@iitgn.ac.in}
}

\begin{document}

\maketitle

\begin{abstract}

Neighborhood-based fairness audits evaluate individual fairness by comparing predictions among similar individuals in feature space. Despite their widespread use, little is known about the robustness of the auditing procedure itself. Because these audits rely on nearest neighbor relationships, small perturbations in feature space can alter local neighborhoods and produce different fairness assessments even when model predictions remain unchanged. We develop a geometric framework for analyzing the robustness of neighborhood-based fairness audits under bounded perturbations. Our analysis establishes sufficient conditions for neighborhood invariance, quantifies how neighborhood replacement propagates to audit instability, and introduces audit volatility, a measure of the expected sensitivity of fairness audits under repeated perturbations. Experiments on benchmark datasets support the theoretical analysis and show that the proposed framework explains the observed stability of neighborhood-based fairness audits.

\end{abstract}

\section{Introduction}

Algorithmic decision-making systems increasingly influence high stakes domains such as criminal justice, finance, healthcare, and public policy, making fairness assessment an essential component of responsible machine learning \cite{Barocas2018FairnessAM,raji2020closing,mitchell2019model}. While much of the fairness literature has focused on group level notions such as demographic parity and equalized odds, these aggregate criteria may conceal unfair treatment of individuals. Individual fairness addresses this limitation by requiring that similar individuals receive similar outcomes \cite{10.1145/2090236.2090255}. Because application specific similarity metrics are rarely available, post hoc fairness auditing methods commonly approximate similarity using $k$-nearest neighbors in the feature space, comparing each individual's prediction with those of its local neighbors \cite{pmlr-v108-xue20a,bellamy2018aifairness360extensible}.

Neighborhood-based fairness audits are widely used because they are model agnostic, computationally efficient, and capable of detecting localized unfairness overlooked by aggregate statistics. However, they rely on a critical assumption that has received little attention: local neighborhoods remain sufficiently stable to serve as reliable proxies for similarity. In practice, feature noise, preprocessing, normalization, missing values, and dataset heterogeneity can alter nearest neighbor relationships even when the prediction model is unchanged. Consequently, identical models may produce different fairness assessments solely because the auditing neighborhood changes, making it difficult to distinguish genuine unfairness from audit instability.

Although robustness has been extensively studied in fair machine learning, existing work primarily focuses on learning prediction models that remain fair under adversarial or feature perturbations \cite{yurochkin2020trainingindividuallyfairml,roh2021sample,hashimoto2018fairness}. In contrast, the robustness of the \emph{fairness audit} itself remains largely unexplored, despite growing recognition that fairness should also be viewed as a measurement problem \cite{10.1145/3442188.3445901,gebru2021datasheetsdatasets}. There is currently no theoretical framework explaining when neighborhood-based fairness audits remain stable under feature space perturbations or how geometric changes in local neighborhoods propagate to audit outcomes.

In this paper, we present a geometric framework for analyzing the robustness of neighborhood-based fairness audits. We formulate local fairness audits as Lipschitz aggregations of pairwise fairness evaluations over local neighborhoods, providing a unified abstraction for a broad class of existing auditing methods. Building on this formulation, we derive deterministic, probabilistic, and expected robustness guarantees that relate bounded feature space perturbations to neighborhood replacement and audit deviation. Our analysis identifies the local separation margin as a geometric certificate for neighborhood preservation and establishes neighborhood replacement as the fundamental mechanism governing audit instability. We further introduce \emph{Audit Volatility}, a measure of expected audit sensitivity under repeated perturbations. Experiments on benchmark datasets validate the theoretical predictions and demonstrate that neighborhood geometry fundamentally governs the robustness of neighborhood-based fairness audits.

Our main contributions are summarized below.

\begin{itemize}
\item We formulate the robustness of neighborhood-based fairness auditing as a new theoretical problem, shifting the focus from robust prediction models to the reliability of the auditing procedure itself.

\item We develop a unified geometric framework that models neighborhood-based fairness audits as Lipschitz aggregations of pairwise fairness evaluations, encompassing a broad class of existing auditing methods.

\item We establish deterministic, probabilistic, and expected robustness guarantees that connect bounded feature space perturbations to neighborhood replacement and audit deviation, identifying neighborhood replacement as the key mechanism governing audit instability.

\item We introduce \emph{Audit Volatility}, a quantitative measure of audit robustness under repeated perturbations, and validate the proposed theory through experiments on multiple benchmark datasets.
\end{itemize}

\section{Related Work}

\paragraph{Individual Fairness and Local Fairness Audits} Individual fairness was introduced by Dwork et al.~\cite{10.1145/2090236.2090255}, who proposed that similar individuals should receive similar decisions with respect to an application specific similarity metric. Subsequent work extended this framework by studying computationally efficient notions of individual fairness \cite{kim2018fairness}, approximately metric fair learning \cite{yona2018probably}, and learning under unknown fairness metrics \cite{gillen2018online}. In practice, however, task specific similarity metrics are rarely available, and fairness is commonly evaluated through post hoc neighborhood-based audits.

Among the most widely used auditing approaches are nearest neighbor consistency measures and FaiTH \cite{pmlr-v108-xue20a}, which evaluate fairness by aggregating pairwise prediction comparisons over local neighborhoods. Toolkits such as AI Fairness 360 \cite{bellamy2018aifairness360extensible} have further popularized these audits for practical fairness evaluation. Although these methods differ in their choice of pairwise fairness functions and aggregation rules, they all rely on nearest neighbor relationships in the feature space. Our framework generalizes this common computational structure and provides a unified analysis of their robustness.

\paragraph{Robustness in Fair Machine Learning} Robustness has become an important objective in fair machine learning, with most existing work focusing on improving the robustness of predictive models. Sensitive Subspace Robustness \cite{yurochkin2020trainingindividuallyfairml} enforces individual fairness under feature perturbations by learning robust representations, while subsequent methods have incorporated robustness into training through sample selection and optimization techniques \cite{roh2021sample,hashimoto2018fairness}. These approaches improve the stability of the predictive model itself. In contrast, we assume that model predictions remain fixed and instead study the robustness of the auditing procedure, isolating instability introduced solely by changes in neighborhood structure.

\paragraph{Robust Statistics and Measurement Reliability} Our work is also closely related to robust statistics, which studies estimators that remain reliable under contaminated observations or outliers \cite{huber2009robust,hampel2001robust,maronna2019robust}. Classical robust aggregation techniques, such as trimmed estimators, motivate the robust aggregation mechanisms studied in this paper. Unlike traditional robust statistics, however, our analysis focuses on instability arising from changes in neighborhood composition rather than contaminated observations.

Finally, recent work has argued that fairness should be viewed as a measurement problem rather than solely an optimization objective \cite{10.1145/3442188.3445901}. Similarly, broader efforts in responsible AI emphasize reliable documentation and auditing practices for machine learning systems \cite{gebru2021datasheetsdatasets,mitchell2019model,raji2020closing}. Complementing these perspectives, we provide the first theoretical framework for analyzing the stability of local fairness audits under feature space perturbations, identifying neighborhood replacement as the fundamental source of audit instability and introducing Audit Volatility as a quantitative robustness measure.


\section{Framework for Neighborhood-based Fairness Audits}

Neighborhood-based fairness audits evaluate whether similar individuals receive similar predictions by aggregating pairwise fairness scores over local neighborhoods. Although existing methods differ in their choice of pairwise fairness functions and aggregation rules, they share this common computational structure. We formalize this abstraction as a unified framework for Neighborhood-based fairness audits, which serves as the foundation for the robustness analysis in Section~4.

\subsection{Local Fairness Audits}

Let $X=\{x_1,\ldots,x_n\}\subseteq(\mathcal{M},d)$ be a dataset embedded in a metric space $(\mathcal{M},d)$, and let $f:X\rightarrow\mathcal{Y}$ be a prediction model. For each individual $x_i$, let $N_k(i)$ denote its set of $k$-nearest neighbors under the metric $d$. A local fairness audit consists of two components: (i) a pairwise fairness function 
$g:\mathcal{Y}\times\mathcal{Y}\rightarrow\mathbb{R},$ which evaluates the fairness relationship between the predictions of neighboring individuals, and (ii) an aggregation operator $\Phi:\mathbb{R}^{k}\rightarrow\mathbb{R},$ which summarizes the resulting vector of pairwise fairness scores.

The local audit score for an individual $x_i$ is defined as $A_i(X)=
\Phi\!\left(
\big(g(f(x_i),f(x_j))\big)_{x_j\in N_k(i)}
\right),$ where the pairwise fairness scores are represented as a $k$-dimensional vector. The corresponding dataset level audit is 
$A(X)=\left(A_1(X),\ldots,A_n(X)\right).$ This abstraction cleanly separates pairwise fairness from aggregation, enabling a unified analysis of Neighborhood-based fairness audits.

\subsection{Lipschitz Aggregation Operators}

The robustness of a local fairness audit depends on the sensitivity of its aggregation operator. To obtain stability guarantees that extend beyond a single auditing method, we consider aggregation operators satisfying a Lipschitz continuity property.

\begin{definition}[Lipschitz Aggregation Operator]
An aggregation operator
$\Phi:\mathbb{R}^{k}\rightarrow\mathbb{R}$
is said to be $L_{\Phi}$-Lipschitz with respect to the $\ell_1$ norm if, for every
$u,v\in\mathbb{R}^{k}$,
\[
|\Phi(u)-\Phi(v)|
\le
L_{\Phi}\|u-v\|_1.
\]
The constant $L_{\Phi}$ quantifies the sensitivity of the aggregation operator to perturbations of its inputs.
\end{definition}

This class includes many aggregation rules used in fairness auditing and robust statistics. For example, the arithmetic mean satisfies $L_{\Phi}=1/k$, while weighted means, trimmed means, Winsorized means, and several robust $M$-estimators are also Lipschitz under mild regularity conditions. Consequently, the theoretical results in Section~4 apply beyond arithmetic averaging to a broad family of neighborhood-based fairness audits.

\paragraph{Examples.}
Consistency based audits instantiate the framework using $g(f(x_i),f(x_j))
=
1-|f(x_i)-f(x_j)|,$ together with arithmetic mean aggregation. Similarly, FaiTH~\cite{pmlr-v108-xue20a} computes pairwise prediction discrepancies over local neighborhoods and summarizes them through an aggregation operator. Consequently, both methods instantiate the proposed framework, and the analysis developed in Section~4 applies uniformly to these and other neighborhood-based fairness audits employing Lipschitz aggregation operators.

The proposed framework identifies two complementary sources of audit robustness: the geometric stability of local neighborhoods and the sensitivity of the aggregation operator. Section~4 shows how bounded feature space perturbations propagate through neighborhood structure and how aggregation sensitivity determines the resulting stability of local fairness audits.

\section{Stability of Local Fairness Audits}

The framework introduced in the previous section expresses a local fairness audit as the aggregation of pairwise fairness scores over local neighborhoods. We now investigate how bounded feature space perturbations affect these audits. Since the prediction model is fixed throughout the analysis, any change in an audit score arises solely through changes in the underlying neighborhood structure.

Our analysis proceeds in three steps. We first introduce a geometric perturbation model and the notion of local separation. We then establish sufficient conditions for neighborhood preservation under bounded perturbations. Finally, we quantify how neighborhood replacement propagates to changes in local fairness audits.

\subsection{Geometric Perturbation Model}

Let $(\mathcal{M},d)$ be a metric space, $X=\{x_1,\ldots,x_n\}\subseteq\mathcal{M}$ a dataset, and $f:X\rightarrow\mathcal{Y}$ a fixed prediction model. For each individual $x_i$, let $N_k(i)$ denote its set of $k$-nearest neighbors under $d$, where ties are resolved deterministically. Since the prediction model remains unchanged, perturbations influence fairness assessments only through changes in neighborhood geometry.

\begin{definition}[$\varepsilon$-Bounded Geometric Perturbation]
\label{def:perturbation}
Let $X'=\{x_1',x_2',\ldots,x_n'\}$ be a perturbed version of $X$. We say that $X'$ is an \emph{$\varepsilon$-bounded geometric perturbation} of $X$ if
$$d(x_i,x_i')\le\varepsilon,
\qquad
\forall\,x_i\in X.$$ For each perturbed point $x_i'$, let $N_k'(i)$ denote its $k$-nearest neighbor set in $X'$.
\end{definition}

For the perturbed dataset $X'$, we analogously denote the corresponding local and dataset level audit scores by $A_i(X')$ and $A(X')$, respectively.

The stability of a local neighborhood depends not only on the perturbation magnitude but also on its geometric separation from the remaining data points. We capture this notion through the following definition.

\begin{definition}[Local Separation Margin]
\label{def:margin}
For every point $x_i\in X$, define
$r_{\max}(i) = \max_{x_j\in N_k(i)} d(x_i,x_j), \qquad r_{\min}(i) = \min_{x_j\notin N_k(i)}
d(x_i,x_j).$ The \emph{local separation margin} of $x_i$ is $\gamma_i = r_{\min}(i)-r_{\max}(i).$
\end{definition}

The local separation margin measures the minimum geometric gap separating an individual's neighborhood from the remaining data points. Larger values of $\gamma_i$ indicate better separated neighborhoods and are therefore expected to yield greater robustness under bounded perturbations. The following subsection shows that this margin provides a sufficient geometric certificate for neighborhood preservation.

\subsection{Geometric Stability of Local Neighborhoods}

We now characterize when local neighborhoods remain invariant under bounded feature space perturbations. We first bound perturbations of pairwise distances, then quantify the resulting change in the local separation margin, and finally derive a sufficient condition for neighborhood preservation.

\begin{lemma}[Pairwise Distance Stability]
\label{lem:distance}
Let $X'$ be an $\varepsilon$-bounded geometric perturbation of $X$. Then, for every pair of points $x_i,x_j\in X$,
\[
\left|
d(x_i',x_j')
-
d(x_i,x_j)
\right|
\le
2\varepsilon.
\]
\end{lemma}

\begin{proof}[Proof Sketch]
The result follows directly from two applications of the triangle inequality: each endpoint may move by at most $\varepsilon$, so the distance between two points changes by at most $2\varepsilon$. The proof is deferred to Appendix~\ref{appendix:A2}A2 due to space constraints.
\end{proof}

Lemma~\ref{lem:distance} bounds the perturbation of every pairwise distance by $2\varepsilon$. Since neighborhood membership depends on the ordering of distances rather than their absolute values, the key quantity is the separation between the farthest neighbor and the nearest non-neighbor. The next lemma quantifies the maximum reduction of this separation margin.

\begin{lemma}[Sharp Bound on Local Separation Margin Reduction]
\label{lem:margin}

For every point $x_i$, let
\[
r_{\max}(i)
=
\max_{x_j\in N_k(i)}
d(x_i,x_j),
\qquad
r_{\min}(i)
=
\min_{x_j\notin N_k(i)}
d(x_i,x_j),
\]
and define the local separation margin $\gamma_i
=
r_{\min}(i)-r_{\max}(i).$ Under an $\varepsilon$-bounded geometric perturbation, 
   
   $\gamma_i-\gamma_i'
\le
4\varepsilon.$ Moreover, the constant $4\varepsilon$ is tight.

\end{lemma}

\begin{proof}[Proof Sketch]
By Lemma~\ref{lem:distance}, distances from $x_i$ to its neighbors can increase by at most $2\varepsilon$, while distances to non-neighbors can decrease by at most $2\varepsilon$. Consequently, the local separation margin decreases by at most $4\varepsilon$. The complete proof, including the tightness construction, is given in \ref{appendix:A2}Appendix~A2.
\end{proof}

Lemma~\ref{lem:margin} shows that the local separation margin can decrease by at most $4\varepsilon$. Therefore, neighborhoods whose original separation margins exceed this threshold remain geometrically distinguishable after perturbation.

\begin{theorem}[Neighborhood Invariance under Bounded Perturbations]
\label{thm:neighborhood}

Let $x_i\in X$ have local separation margin
\[
\gamma_i
=
r_{\min}(i)-r_{\max}(i).
\]

If
\[
\gamma_i>4\varepsilon,
\]
then every original neighbor of $x_i$ remains strictly closer to $x_i'$ than every original non-neighbor. Consequently,
\[
N_k'(i)=N_k(i).
\]

\end{theorem}

\begin{proof}[Proof Sketch]
Lemma~\ref{lem:margin} guarantees that the local separation margin remains positive whenever $\gamma_i>4\varepsilon$. Consequently, every original neighbor remains strictly closer than every original non-neighbor, preserving the ordering of the $k$ nearest neighbors and hence the neighborhood itself. The complete proof is deferred to \ref{appendix:A2}Appendix~A2. 
\end{proof}

Theorem~\ref{thm:neighborhood} identifies the local separation margin as a sufficient geometric certificate for neighborhood preservation under bounded perturbations. Together, Lemmas~\ref{lem:distance},~\ref{lem:margin}, and Theorem~\ref{thm:neighborhood} establish the geometric foundation of our framework: bounded perturbations induce bounded distance distortions, bounded distance distortions limit the reduction of the local separation margin, and sufficiently separated neighborhoods remain invariant. This characterization forms the basis for the deterministic, probabilistic, and expected stability guarantees developed in the following subsections.

\subsection{Probabilistic Robustness under Random Perturbations}

Theorem~\ref{thm:neighborhood} provides a deterministic guarantee for neighborhood preservation under bounded perturbations. In practice, however, feature perturbations often arise from stochastic sources such as measurement noise or randomized preprocessing. We therefore quantify the probability that a local neighborhood changes under random perturbations.

\begin{corollary}[Probabilistic Neighborhood Stability]
\label{cor:probabilistic}

Suppose the perturbed dataset $X'$ is sampled from a probability distribution over admissible perturbations, and let
\[
R=\max_{1\le i\le n} d(x_i,x_i')
\]
denote the random perturbation radius. Then, for every point $x_i$,
\[
\Pr\!\left(N_k'(i)\neq N_k(i)\right)
\le
\Pr\!\left(R\ge\frac{\gamma_i}{4}\right).
\]

\end{corollary}

\begin{proof}[Proof Sketch]
By Theorem~\ref{thm:neighborhood}, the neighborhood of $x_i$ is preserved whenever $R<\gamma_i/4$. Thus, neighborhood replacement can occur only if $R\ge\gamma_i/4$, yielding the stated probability bound. The complete proof is provided in \ref{appendix:A3} Appendix~A3.
\end{proof}

Corollary~\ref{cor:probabilistic} separates neighborhood robustness into two independent components: the local separation margin $\gamma_i$, which depends solely on the geometry of the dataset, and the perturbation radius $R$, which captures the stochastic perturbation process. Consequently, any concentration inequality or tail bound on $R$ immediately yields a corresponding probabilistic guarantee for neighborhood preservation. This distribution independent characterization extends the deterministic guarantee of Theorem~\ref{thm:neighborhood} to a broad class of stochastic perturbation models and provides the probabilistic foundation for analyzing the stability of local fairness audits.

\subsection{Stability of Local Fairness Audits}

Theorem~\ref{thm:neighborhood} establishes that sufficiently separated neighborhoods remain invariant under bounded perturbations. We now characterize how neighborhood changes affect local fairness audits. Since the prediction model is fixed, perturbations influence the audit solely through changes in the neighborhood participating in the aggregation.

\begin{corollary}[Audit Invariance]
\label{cor:invariance}

Let
\[
A_i(X)=\Phi(\mathbf g_i(X))
\]
be a local fairness audit. If
\[
N_k'(i)=N_k(i),
\]
then
\[
A_i(X)=A_i(X').
\]

\end{corollary}

\begin{proof}[Proof Sketch]
Neighborhood preservation implies that the same neighbors contribute identical pairwise fairness scores before and after perturbation. Since the prediction model is fixed, the fairness-score vectors coincide, yielding identical audit scores. The full proof can be found in \ref{appendix:A4}Appendix~A4.
\end{proof}

When neighborhoods change, audit stability depends on the extent of neighborhood replacement.

\begin{definition}[Neighborhood Replacement Set]
\label{def:replacement}

For an individual $x_i$, define
\[
R_i
=
N_k(i)\triangle N_k'(i),
\]
where $\triangle$ denotes the symmetric difference.

\end{definition}

\begin{theorem}[Audit Stability via Neighborhood Replacement]
\label{thm:replacement}

Let
\[
A_i(X)=\Phi(\mathbf g_i(X))
\]
be a local fairness audit, where
\[
\Phi:\mathbb{R}^{k}\rightarrow\mathbb{R}
\]
is an $L_{\Phi}$-Lipschitz aggregation operator. Suppose the prediction model is fixed, the pairwise fairness scores satisfy
\[
|g_{ij}|\le M,
\]
and let
\[
R_i=N_k(i)\triangle N_k'(i)
\]
denote the neighborhood replacement set. Then
\[
|A_i(X)-A_i(X')|
\le
L_{\Phi}M|R_i|.
\]

\end{theorem}

\begin{proof}[Proof Sketch]
Common neighbors contribute identical pairwise fairness scores before and after perturbation because the prediction model is fixed. Thus, only neighbors in the replacement set contribute to changes in the fairness-score vector, with each replaced neighbor contributing at most $M$ to its $\ell_1$ difference. The result then follows immediately from the Lipschitz continuity of $\Phi$. The complete proof is provided in \ref{appendix:A4}Appendix~A4.
\end{proof}

\begin{corollary}[Mean Aggregation]
\label{cor:mean}

For arithmetic mean aggregation,
\[
\Phi(z)
=
\frac1k
\sum_{i=1}^{k}z_i,
\]
the Lipschitz constant is
\[
L_{\Phi}=\frac1k.
\]
Consequently,
\[
|A_i(X)-A_i(X')|
\le
\frac{M}{k}|R_i|.
\]

\end{corollary}

\begin{proof}[Proof Sketch]
The arithmetic mean is $1/k$-Lipschitz with respect to the $\ell_1$ norm. Substituting $L_{\Phi}=1/k$ into Theorem~\ref{thm:replacement} yields the stated bound. The complete proof is provided in \ref{appendix:A4}Appendix~A4.
\end{proof}

Theorem~\ref{thm:replacement} establishes that the robustness of a local fairness audit is governed by two complementary factors: the sensitivity of the aggregation operator, captured by its Lipschitz constant, and the extent of neighborhood replacement induced by feature space perturbations. Combined with Theorem~\ref{thm:neighborhood}, it completes the robustness pipeline underlying our framework: bounded feature space perturbations induce controlled neighborhood replacement, and the resulting audit deviation is bounded linearly by the number of replaced neighbors. The arithmetic mean follows as an immediate special case, while the same analysis applies to weighted means, trimmed means, and other Lipschitz aggregation operators.

\subsection{Audit Volatility}

The preceding results characterize the worst case deviation of a local fairness audit under a single perturbation. In practice, however, robustness is typically assessed over repeated perturbations drawn from a stochastic perturbation model. We therefore introduce \emph{audit volatility}, which measures the expected variation of a local fairness audit across perturbation realizations.

\begin{definition}[Audit Volatility]
\label{def:volatility}

Let $\mathcal{P}$ denote a probability distribution over admissible perturbations. The \emph{audit volatility} of an individual $x_i$ is
\[
\kappa_i
=
\mathbb{E}_{X'\sim\mathcal{P}}
\!\left[
|A_i(X)-A_i(X')|
\right].
\]

The corresponding dataset level audit volatility is
\[
\kappa
=
\frac{1}{n}
\sum_{i=1}^{n}
\kappa_i.
\]

\end{definition}

\begin{theorem}[Audit Volatility Bound]
\label{thm:volatility}

Under the assumptions of Theorem~\ref{thm:replacement},
\[
\kappa_i
\le
L_\Phi M
\,
\mathbb E
\!\left[
|R_i|
\right].
\]

In particular, for arithmetic mean aggregation,
\[
\kappa_i
\le
\frac{M}{k}
\,
\mathbb E
\!\left[
|R_i|
\right].
\]

\end{theorem}

\begin{proof}[Proof Sketch]
The result follows by taking expectations on both sides of Theorem~\ref{thm:replacement} and applying the linearity of expectation. The complete proof is provided in \ref{appendix:A5}Appendix~A5.
\end{proof}

Theorem~\ref{thm:volatility} shows that audit volatility is jointly determined by the sensitivity of the aggregation operator and the expected amount of neighborhood replacement induced by the perturbation model. Together with Corollary~\ref{cor:probabilistic} and Theorem~\ref{thm:replacement}, it completes the theoretical framework developed in this section: the local separation margin governs neighborhood preservation under perturbations, neighborhood replacement quantifies the resulting deviation of a local fairness audit, and the expected amount of neighborhood replacement determines audit volatility under stochastic perturbations.

\section{Experimental Evaluation}

We evaluate whether the proposed geometric framework accurately explains the robustness of Neighborhood-based fairness audits under feature perturbations. Specifically, we address four questions: (i) do bounded perturbations affect neighborhood geometry as predicted by the theory; (ii) does neighborhood replacement explain audit instability; (iii) how does the aggregation operator influence robustness; and (iv) how does the intrinsic geometry of a dataset affect audit stability?

\subsection{Experimental Setup}

We evaluate the proposed framework on the Adult Income, Bank Marketing, and COMPAS datasets obtained through the AI Fairness 360 toolkit and UCI~\cite{aif360-oct-2018,UCI}. Continuous attributes are standardized, categorical features are one-hot encoded, and a logistic regression classifier is trained using the default \texttt{scikit-learn} implementation. Unless otherwise stated, local fairness audits use Euclidean $k$-nearest neighborhoods with $k=10$, prediction consistency as the pairwise fairness score, and arithmetic mean aggregation.

Perturbed datasets are generated according to Definition~2. Unless otherwise specified, perturbations are sampled uniformly from an $\ell_2$ ball of radius $\epsilon$. To evaluate the distribution independence of the proposed theory, we additionally consider spherical, clipped Gaussian, and sparse coordinate perturbations under the same perturbation budget.

Throughout the experiments, we report the Neighborhood Preservation Rate (NPR), the expected neighborhood replacement size $\mathbb{E}[|R_i|]$, and the dataset-level audit volatility

\[
\kappa=
\frac1n
\sum_{i=1}^{n}
\mathbb{E}
\!\left[
|A_i(X)-A_i(X')|
\right].
\]

All reported values are averaged over 30 Monte Carlo perturbation trials.

\subsection{Empirical Validation of Geometric Stability}

We first validate the geometric results established in Lemmas~1 and~2 and Theorem~3. These results characterize how bounded perturbations affect pairwise distances, local separation margins, and neighborhood preservation.

\begin{table}[t]
\centering
\caption{Empirical validation of the geometric stability bounds established in Lemmas~1 and~2 under $\ell_2$ perturbations with $\epsilon=0.05$. Across all datasets, the observed maxima remain below the corresponding theoretical upper bounds.}
\label{tab:lemma_validation}
\small
\begin{tabular}{lcccc}
\toprule
Dataset &
$\max |d'-d|$ &
Theory($2\epsilon$) &
$\max(\gamma-\gamma')$ &
Theory($4\epsilon$)\\
\midrule
Adult  & 0.033 & 0.100 & 0.032 & 0.200\\
Bank   & 0.038 & 0.100 & 0.047 & 0.200\\
COMPAS & 0.066 & 0.100 & 0.025 & 0.200\\
\bottomrule
\end{tabular}
\end{table}

For each dataset, we generate perturbations uniformly within an $\ell_2$ ball of radius $\epsilon=0.05$ and measure (i) the maximum pairwise distance distortion,
$\max_{i,j}|d'(x_i,x_j)-d(x_i,x_j)|$, and (ii) the maximum reduction in the local separation margin,
$\max_i(\gamma_i-\gamma_i')$.
Table~\ref{tab:lemma_validation} shows that the observed maxima remain well below the theoretical upper bounds of $2\epsilon$ and $4\epsilon$, providing empirical support for Lemmas~1 and~2.

\begin{figure}[t]
    \centering
    \includegraphics[width=\linewidth]{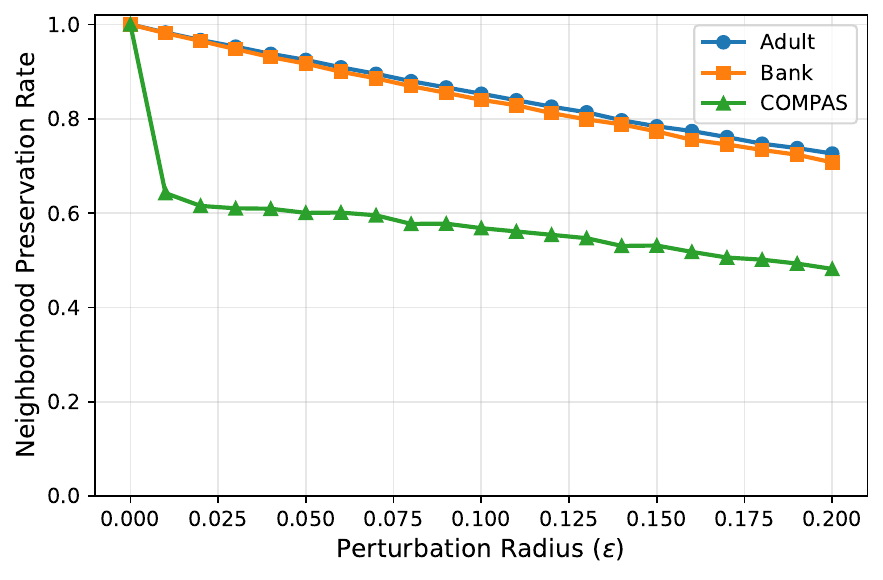}
    \caption{Neighborhood Preservation Rate (NPR) as a function of the perturbation radius $\epsilon$. Larger perturbations progressively reduce neighborhood stability, consistent with Theorem~3.}
    \label{fig:npr}
\end{figure}

We next evaluate Theorem~3 by measuring the Neighborhood Preservation Rate (NPR) as the perturbation radius increases. Figure~\ref{fig:npr} shows a monotonic decrease in NPR across all datasets, indicating that neighborhood stability degrades predictably with perturbation magnitude. Adult and Bank remain substantially more stable than COMPAS, reflecting larger local separation margins. Together, these results validate the geometric foundations of the proposed framework and identify the local separation margin as an effective certificate for neighborhood preservation.

\subsection{Neighborhood Replacement Explains Audit Instability}

Our theory identifies neighborhood replacement as the fundamental mechanism governing audit instability. We evaluate this hypothesis by measuring the expected neighborhood replacement $\mathbb{E}[|R_i|]$ and the resulting audit volatility under increasing perturbation magnitudes and multiple perturbation models.

\begin{figure}[t]
    \centering
    \includegraphics[width=\linewidth]{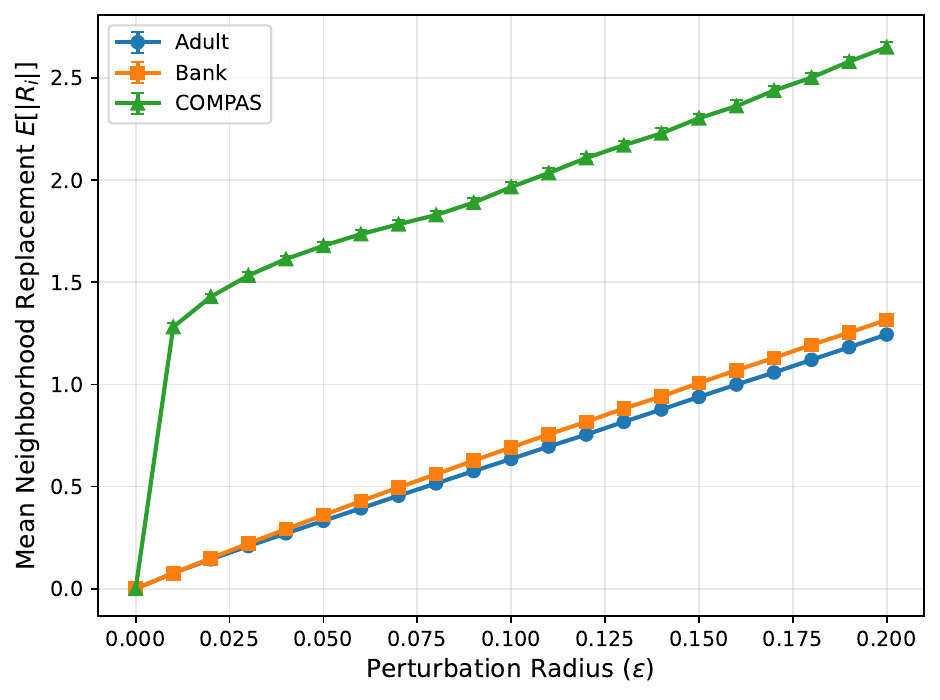}
    \caption{Expected neighborhood replacement $\mathbb{E}[|R_i|]$ as a function of the perturbation radius $\epsilon$. Neighborhood replacement increases monotonically across all datasets, with COMPAS exhibiting substantially greater geometric sensitivity than Adult and Bank.}
    \label{fig:replacement}
\end{figure}

Figure~\ref{fig:replacement} shows that neighborhood replacement increases monotonically with the perturbation radius for all datasets. Adult and Bank exhibit relatively gradual growth, whereas COMPAS consistently experiences substantially larger neighborhood replacement, reflecting its smaller local separation margins. These results provide empirical support for the geometric characterization developed in Section~4.

\begin{figure}[t]
    \centering
    \includegraphics[width=\linewidth]{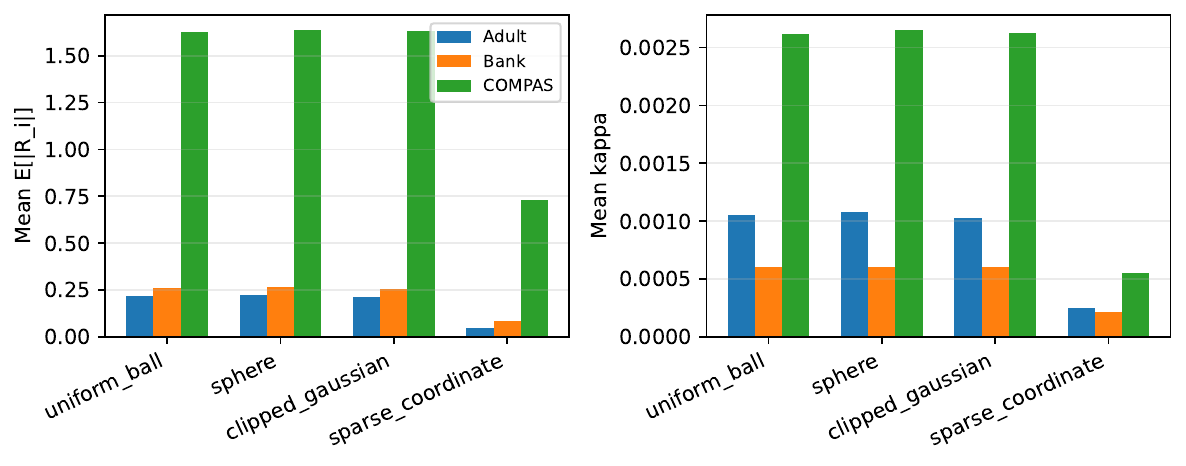}
    \caption{Neighborhood replacement (left) and audit volatility (right) under four perturbation models with identical perturbation budgets. Despite differences in the perturbation distributions, audit volatility consistently follows neighborhood replacement across all datasets.}
    \label{fig:perturbation_models}
\end{figure}

Figure~\ref{fig:perturbation_models} further demonstrates that audit volatility closely tracks neighborhood replacement under four distinct perturbation models. Although the perturbation distributions differ substantially, the observed audit volatility is primarily determined by the amount of neighborhood replacement, supporting the distribution independent characterization established by Theorem~\ref{thm:volatility}.

For arithmetic mean aggregation, Theorem~\ref{thm:volatility} predicts
\[
\kappa_i
\le
\frac{M}{k}\,
\mathbb{E}[|R_i|],
\]
where $M=1$ for the bounded consistency score.

\begin{figure}[t]

\caption{Empirical validation of Theorem~\ref{thm:volatility}. Top: audit volatility increases with the perturbation radius. Bottom: after normalization by the theoretical upper bound $(M/k)\mathbb{E}[|R_i|]$, all empirical values remain below one across Adult, Bank, and COMPAS.}
\label{fig:theorem9-validation}
\end{figure}

Figure~\ref{fig:theorem9-validation} validates the theoretical bound over increasing perturbation radii. Audit volatility grows monotonically with perturbation magnitude while remaining well below the predicted upper bound across all datasets, with no observed violations. Together, these experiments demonstrate that neighborhood replacement provides an accurate and predictive explanation for the robustness of Neighborhood-based fairness audits.

\subsection{Aggregation Robustness}

Theorem~6 predicts that audit stability depends on the Lipschitz constant of the aggregation operator. We evaluate this prediction by comparing four representative aggregators: arithmetic mean, $20\%$ trimmed mean, median, and worst neighbor aggregation under identical perturbations.

\begin{figure}[t]
    \centering
    \includegraphics[width=\linewidth]{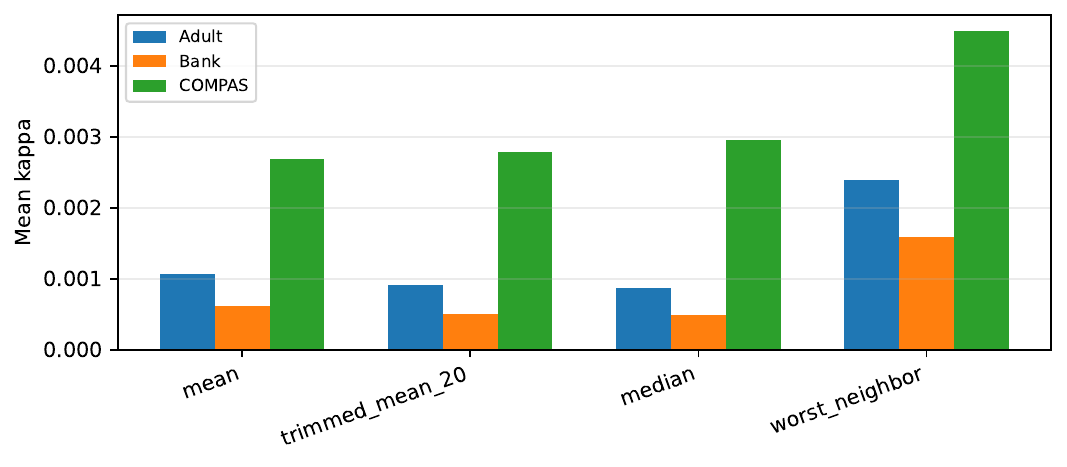}
    \caption{Audit volatility under different aggregation operators. Robust aggregators (trimmed mean and median) consistently exhibit lower volatility than the worst neighbor operator, supporting Theorem~6.}
    \label{fig:aggregation_sensitivity}
\end{figure}

Figure~\ref{fig:aggregation_sensitivity} shows that aggregation choice has a direct impact on audit robustness across all datasets. The trimmed mean and median consistently produce lower audit volatility than the worst neighbor operator, while the arithmetic mean exhibits intermediate behavior. These observations are consistent with Theorem~6, confirming that aggregation sensitivity is a key determinant of the robustness of Neighborhood-based fairness audits.
\subsection{Effect of Dataset Geometry}

Finally, we investigate how the intrinsic geometry of a dataset influences audit robustness. Since Theorem~3 identifies the local separation margin as the geometric certificate for neighborhood preservation, we examine its relationship with neighborhood replacement and audit volatility.

\begin{figure}[t]
    \centering
    \includegraphics[width=\linewidth]{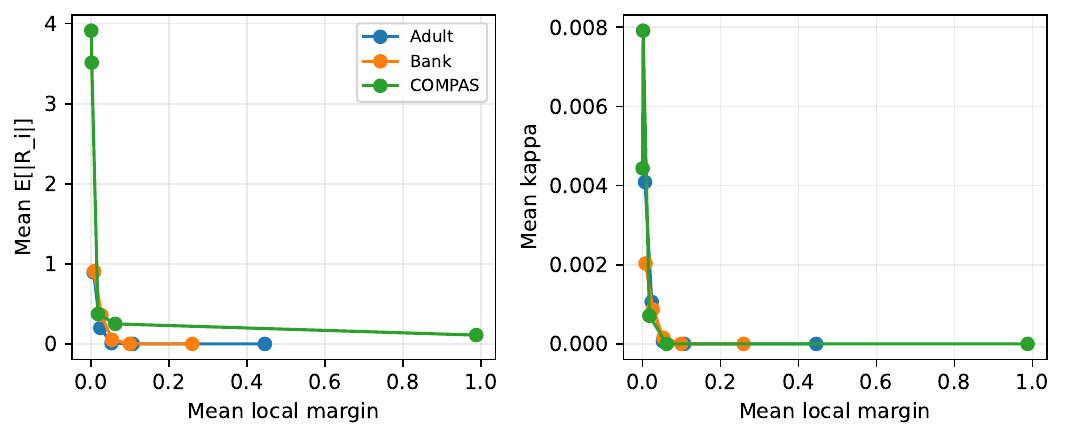}
    \caption{Effect of the local separation margin on audit robustness. \textbf{Left:} Expected neighborhood replacement $\mathbb{E}[|R_i|]$ decreases with increasing separation margin. \textbf{Right:} Audit volatility $\kappa$ exhibits the same trend, approaching zero for well-separated neighborhoods.}
    \label{fig:margin_stratification}
\end{figure}

Figure~\ref{fig:margin_stratification} shows that larger local separation margins consistently produce smaller neighborhood replacement and, consequently, lower audit volatility across all datasets. For well separated neighborhoods, audit volatility approaches zero, indicating that fairness assessments become increasingly insensitive to bounded perturbations. These observations provide empirical support for Theorem~3, confirming that the local separation margin is the primary geometric factor governing the robustness of Neighborhood-based fairness audits.

\section{Conclusion}

We presented a geometric framework for analyzing the robustness of Neighborhood-based fairness audits under bounded feature space perturbations. By modeling local fairness audits as Lipschitz aggregations of pairwise fairness evaluations, we established deterministic, probabilistic, and expected stability guarantees that relate feature space perturbations to neighborhood replacement and audit deviation. Our analysis identified the local separation margin as a geometric certificate for neighborhood preservation and introduced audit volatility as a quantitative measure of audit robustness. Experimental results on benchmark datasets validated the theoretical predictions, demonstrating that neighborhood replacement provides the key geometric mechanism governing audit stability. Together, these results establish a principled foundation for the design and analysis of robust Neighborhood-based fairness auditing methods.

\bibliography{AuthorKit27/AuthorKit27/arxiv}
\bibliographystyle{plain}


\newpage
\appendix

\section{Appendix A : Stability of Local Fairness Audits}
\label{appendix:A}

The general framework introduced in the previous section expresses a local fairness audit as the composition of two components: a pairwise fairness function that quantifies fairness relationships between neighboring individuals and an aggregation operator that summarizes these local relationships into an individual audit score. This abstraction naturally raises the following question:

\begin{quote}
\emph{How sensitive are local fairness audits to small perturbations of the underlying feature space?}
\end{quote}

Answering this question is essential for understanding the reliability of Neighborhood-based fairness assessments. Since the prediction model is fixed throughout the audit, perturbations affect fairness scores only indirectly through changes in the geometric structure of the feature space. Our analysis therefore separates the effect of feature-space perturbations from the auditing procedure itself and studies how geometric changes propagate to local fairness assessments.

The theoretical development proceeds in three stages. We first formalize a geometric perturbation model and introduce the notion of local separation. We then establish sufficient conditions under which local neighborhoods remain invariant under bounded perturbations. Finally, we quantify how neighborhood changes affect the stability and volatility of local fairness audits.

\subsection{Appendix A2: Geometric Stability of Local Neighborhoods}
\label{appendix:A2}
The perturbation model introduced in the previous subsection specifies how the feature representation of each individual may vary under bounded perturbations. Since local fairness audits are computed from $k$-nearest neighbor sets, understanding the stability of these neighborhoods is a prerequisite for analyzing the robustness of the resulting fairness assessments. Our analysis proceeds in three steps. We first bound the perturbation of pairwise distances, then characterize the maximum reduction of the local separation margin, and finally derive a sufficient condition guaranteeing the invariance of local neighborhoods under bounded perturbations.



Here we present proof of lemma~1.

\begin{proof}

By the triangle inequality, $d(x_i',x_j')
\le
d(x_i',x_i)
+
d(x_i,x_j)
+
d(x_j,x_j').$ Since $d(x_i',x_i)\le\varepsilon,
\qquad
d(x_j',x_j)\le\varepsilon,$ we obtain

$d(x_i',x_j')
\le
d(x_i,x_j)+2\varepsilon.$

Similarly, $d(x_i,x_j)
\le
d(x_i,x_i')
+
d(x_i',x_j')
+
d(x_j',x_j),$ which implies $d(x_i',x_j')
\ge
d(x_i,x_j)-2\varepsilon.$ Combining the two inequalities yields

$
\left|
d(x_i',x_j')
-
d(x_i,x_j)
\right|
\le
2\varepsilon,$ proving the lemma.

\end{proof}

Lemma~\ref{lem:distance} shows that every pairwise distance changes by at most $2\varepsilon$ under the perturbation model. Neighborhood membership, however, depends on the relative ordering of distances rather than on their absolute values. The relevant quantity is therefore the separation margin between the farthest neighbor and the nearest non-neighbor. The next lemma shows that this margin can decrease by at most $4\varepsilon$.

\subsection{Proof of lemma~2.}














\begin{proof}

Fix any point $x_j\notin N_k(i)$. By Lemma~\ref{lem:distance},

\[
d(x_i',x_j')
\ge
d(x_i,x_j)-2\varepsilon.
\]

Taking the minimum over all points outside the neighborhood gives

\[
r_{\min}'(i)
=
\min_{x_j\notin N_k(i)}
d(x_i',x_j')
\ge
r_{\min}(i)-2\varepsilon.
\]

Similarly, fix any point $x_j\in N_k(i)$. Again by Lemma~\ref{lem:distance},

\[
d(x_i',x_j')
\le
d(x_i,x_j)+2\varepsilon.
\]

Taking the maximum over all neighbors gives

\[
r_{\max}'(i)
=
\max_{x_j\in N_k(i)}
d(x_i',x_j')
\le
r_{\max}(i)+2\varepsilon.
\]

Therefore,

\[
\begin{aligned}
\gamma_i'
&=
r_{\min}'(i)-r_{\max}'(i)\\
&\ge
\left(r_{\min}(i)-2\varepsilon\right)
-
\left(r_{\max}(i)+2\varepsilon\right)\\
&=
\gamma_i-4\varepsilon.
\end{aligned}
\]

Equivalently,

\[
\gamma_i-\gamma_i'
\le
4\varepsilon.
\]

To show that the bound is attained, consider the case $k=1$ with three collinear points
$a_i,x_i,b_i$, where
$a_i$ is the unique neighbor of $x_i$ and
$b_i$ is the unique non-neighbor.
Move \(a_i\) away from \(x_i\), move \(b_i\) toward \(x_i\), and move \(x_i\) toward \(b_i\), each by distance \(\varepsilon\). Then

\[
r_{\max}'
=
r_{\max}+2\varepsilon,
\]

and

\[
r_{\min}'
=
r_{\min}-2\varepsilon.
\]

Consequently,

\[
\gamma_i-\gamma_i'
=
4\varepsilon,
\]

showing that the bound is tight.

\end{proof}

Lemma~\ref{lem:margin} identifies $4\varepsilon$ as the largest possible reduction of the local separation margin under the perturbation model. Consequently, whenever the original separation margin exceeds this worst-case reduction, the ordering between neighbors and non-neighbors is preserved.













\subsection{Proof of Theorem 3.}
\begin{proof}

Fix arbitrary points

\[
x_a\in N_k(i),
\qquad
x_b\notin N_k(i).
\]

By Definition~\ref{def:margin},

\[
d(x_i,x_a)\le r_{\max}(i),
\qquad
d(x_i,x_b)\ge r_{\min}(i),
\]

which implies

\[
d(x_i,x_b)-d(x_i,x_a)
\ge
r_{\min}(i)-r_{\max}(i)
=
\gamma_i.
\]

By Lemma~2,

\[
d(x_i',x_a')
\le
d(x_i,x_a)+2\varepsilon,
\]

and

\[
d(x_i',x_b')
\ge
d(x_i,x_b)-2\varepsilon.
\]

Therefore,

\[
\begin{aligned}
d(x_i',x_b')-d(x_i',x_a')
&\ge
\left(d(x_i,x_b)-2\varepsilon\right)
-
\left(d(x_i,x_a)+2\varepsilon\right)\\
&=
d(x_i,x_b)-d(x_i,x_a)-4\varepsilon\\
&\ge
\gamma_i-4\varepsilon.
\end{aligned}
\]

Since

\[
\gamma_i>4\varepsilon,
\]

it follows that

\[
d(x_i',x_b')-d(x_i',x_a')>0,
\]

or equivalently,

\[
d(x_i',x_a')
<
d(x_i',x_b').
\]

Because the choice of $x_a$ and $x_b$ was arbitrary, the above inequality holds for every original neighbor and every original non-neighbor. Thus, every point in $N_k(i)$ remains strictly closer to $x_i'$ than every point outside $N_k(i)$ after perturbation.

Since both $N_k(i)$ and $N_k'(i)$ contain exactly $k$ points, the first $k$ positions in the perturbed distance ordering are occupied precisely by the points in $N_k(i)$. Therefore,

\[
N_k'(i)=N_k(i).
\]

\end{proof}

Theorem~\ref{thm:neighborhood} identifies the local separation margin as a geometric certificate for neighborhood preservation. The proof shows that whenever the separation margin exceeds the maximum perturbation-induced distortion of pairwise distances, every original neighbor remains strictly closer than every original non-neighbor after perturbation, thereby preserving the $k$-nearest neighbor set. Lemma~\ref{lem:margin} further establishes that the threshold $4\varepsilon$ is intrinsic to the perturbation model: the local separation margin can decrease by at most $4\varepsilon$, and this bound is tight. Consequently, neighborhoods with separation margins exceeding this threshold remain invariant under every admissible perturbation. Since the proposed fairness audits depend exclusively on local neighborhoods, Theorem~\ref{thm:neighborhood} provides the geometric bridge between feature-space perturbations and audit robustness, forming the foundation for the deterministic, probabilistic, and expected stability guarantees developed in the following subsections.

\begin{remark}[Optimality of the Stability Threshold]
The condition $\gamma_i>4\varepsilon$ in Theorem~3 is
essentially optimal. By Lemma~2, the local separation margin
may decrease by exactly $4\varepsilon$ under an admissible
$\varepsilon$-bounded perturbation, and this bound is tight.
Therefore, any deterministic neighborhood-preservation
guarantee that holds uniformly over all admissible
perturbations cannot replace the threshold $4\varepsilon$ by a
smaller universal constant. Consequently, the stability
condition in Theorem~3 is intrinsic to the perturbation model
rather than an artifact of the analysis.
\end{remark}

\subsection{Appendix A3: Probabilistic Robustness under Random Perturbations}
\label{appendix:A3}
Theorem~\ref{thm:neighborhood} establishes a deterministic
criterion guaranteeing neighborhood preservation under every
admissible perturbation of bounded magnitude. In many
practical settings, however, feature perturbations arise from
stochastic sources such as measurement noise, sensor
uncertainty, or randomized preprocessing. Rather than asking
whether a neighborhood is preserved for every admissible
perturbation, it is therefore natural to quantify the
probability that a neighborhood changes under a prescribed
perturbation model.

The deterministic guarantee immediately yields the following
probabilistic robustness result.








\subsection{Proof of Corollary~4.}
\begin{proof}

Let

\[
E_i
=
\{N_k'(i)\neq N_k(i)\}
\]

denote the event that the neighborhood of $x_i$ changes.

By Theorem~\ref{thm:neighborhood},

\[
R
<
\frac{\gamma_i}{4}
\]

implies

\[
N_k'(i)=N_k(i).
\]

Equivalently,

\[
E_i
\subseteq
\left\{
R
\ge
\frac{\gamma_i}{4}
\right\}.
\]

Taking probabilities and using the monotonicity of probability,

\[
\Pr(E_i)
\le
\Pr\!\left(
R
\ge
\frac{\gamma_i}{4}
\right),
\]

which proves the result.

\end{proof}

\begin{remark}

Corollary~\ref{cor:probabilistic} separates neighborhood
robustness into geometric and stochastic components.
The local separation margin $\gamma_i$ depends solely on the
geometry of the dataset, whereas the random variable $R$
captures the uncertainty introduced by the perturbation
process. Consequently, any concentration inequality or tail
bound $R$ immediately yields a corresponding
probabilistic guarantee for neighborhood preservation.

This separation is particularly useful because the geometric
analysis developed in Theorem~\ref{thm:neighborhood} is
independent of the choice of perturbation distribution.
Therefore, the same deterministic argument applies uniformly
to a broad family of stochastic perturbation models by simply
specializing the corresponding tail bound for $R$.

The probabilistic interpretation developed above serves as the
bridge between deterministic neighborhood stability and the
audit robustness guarantees established in the following
subsection.

\end{remark}

\subsection{Appendix A4: Stability of Local Fairness Audits}
\label{appendix:A4}
Theorem~\ref{thm:neighborhood} established that sufficiently separated
local neighborhoods remain invariant under bounded perturbations. We now
investigate how this geometric stability translates into the stability
of local fairness audits. Since a local audit score is obtained by
aggregating pairwise fairness evaluations over an individual's
neighborhood, perturbations influence the audit only through changes in
the set of neighboring individuals participating in the aggregation. We
first establish the ideal case where the neighborhood remains unchanged
and then quantify the effect of neighborhood replacement on the audit
score.








\subsection{Proof of Corollary 5.}

\begin{proof}

If

\[
N_k'(i)=N_k(i),
\]

the same neighboring individuals contribute to the audit before and
after perturbation. Since the prediction model is fixed, every pairwise
fairness evaluation remains unchanged. Hence,

\[
\mathbf g_i(X)=\mathbf g_i(X'),
\]

which immediately implies

\[
A_i(X)
=
\Phi(\mathbf g_i(X))
=
\Phi(\mathbf g_i(X'))
=
A_i(X').
\]

\end{proof}

While Corollary~\ref{cor:invariance} guarantees exact stability whenever
neighborhoods are preserved, perturbations may replace some neighbors
when the local separation margin is insufficient. We therefore quantify
the resulting change in the audit score in terms of the number of
replaced neighbors.








The general framework introduced in Section~3 represents a local
fairness audit as the composition of a vector of pairwise fairness
evaluations and an aggregation operator. We first establish a general
stability result that depends only on the aggregation operator and is
independent of the source of perturbation. We then specialize this
result to geometric perturbations by relating changes in the audit score
to neighborhood replacement.

Let

\[
\mathbf g_i(X)
=
(g_{ij})_{x_j\in N_k(i)}
\]

denote the vector of pairwise fairness scores associated with
individual $x_i$, where

\[
|g_{ij}|\le M
\]

for some constant $M>0$.
\subsection{Proof of Theorem 6: Audit Stability via Neighborhood Replacement}

When the neighborhood changes under perturbation, the pairwise fairness
scores before and after perturbation are evaluated over different
neighbor sets. To compare the resulting fairness-score vectors, we match
entries according to the identities of their associated neighbors.
Consequently, pairwise fairness scores corresponding to common neighbors
are compared directly, while differences arise only from neighbors that
The following lemma isolates the contribution of the aggregation
operator independently of the geometric changes in the neighborhood.

\begin{lemma}[Lipschitz Stability of the Aggregation Operator]
\label{lem:lipschitz}

Let
\[
A_i(X)=\Phi(\mathbf g_i(X)),
\]
where
\[
\Phi:\mathbb R^k\rightarrow\mathbb R
\]
is an $L_\Phi$-Lipschitz aggregation operator with respect to the
$\ell_1$ norm. Then

\[
|A_i(X)-A_i(X')|
\le
L_\Phi
\|
\mathbf g_i(X)-\mathbf g_i(X')
\|_1.
\]

\end{lemma}

\begin{proof}

Since

\[
A_i(X)=\Phi(\mathbf g_i(X)),
\qquad
A_i(X')=\Phi(\mathbf g_i(X')),
\]

the Lipschitz continuity of $\Phi$ immediately gives

\[
\begin{aligned}
|A_i(X)-A_i(X')|
&=
|\Phi(\mathbf g_i(X))
-
\Phi(\mathbf g_i(X'))|
\\
&\le
L_\Phi
\|
\mathbf g_i(X)-\mathbf g_i(X')
\|_1.
\end{aligned}
\]

\end{proof}

The next lemma shows that the difference between the fairness-score
vectors is completely determined by neighborhood replacement.

\begin{lemma}[Neighborhood Replacement Bound]
\label{lem:replacement}

Assume that

\begin{enumerate}
\item the prediction function $f$ is fixed;
\item the pairwise fairness scores satisfy
\[
|g_{ij}|\le M;
\]
\item
\[
R_i=N_k(i)\triangle N_k'(i)
\]
is the neighborhood replacement set.
\end{enumerate}

Then

\[
\|
\mathbf g_i(X)-\mathbf g_i(X')
\|_1
\le
M|R_i|.
\]

\end{lemma}

\begin{proof}

Compare the fairness-score vectors by matching entries corresponding to
the same neighbor whenever that neighbor belongs to both
\(N_k(i)\) and \(N_k'(i)\).

For every common neighbor in

\[
N_k(i)\cap N_k'(i),
\]

the prediction model is unchanged. Hence the associated pairwise
fairness score is identical before and after perturbation, and its
contribution to the $\ell_1$ difference is zero.

Therefore, only neighbors belonging to the replacement set

\[
R_i
=
N_k(i)\triangle N_k'(i)
\]

can contribute to

\[
\|
\mathbf g_i(X)-\mathbf g_i(X')
\|_1.
\]

Since each pairwise fairness score is bounded by

\[
|g_{ij}|\le M,
\]

every replaced neighbor contributes at most \(M\) to the
$\ell_1$ difference. As the replacement set contains
\(|R_i|\) neighbors,

\[
\|
\mathbf g_i(X)-\mathbf g_i(X')
\|_1
\le
M|R_i|.
\]

\end{proof}

\begin{proof}[Proof of Theorem~6]

By Lemma~\ref{lem:lipschitz},

\[
|A_i(X)-A_i(X')|
\le
L_\Phi
\|
\mathbf g_i(X)-\mathbf g_i(X')
\|_1.
\]

Applying Lemma~\ref{lem:replacement} yields

\[
|A_i(X)-A_i(X')|
\le
L_\Phi M|R_i|,
\]

which is exactly the statement of Theorem~6.

\end{proof}









\subsection{Proof of Corollary 7}

\begin{proof}

For the arithmetic mean aggregation,

\[
\Phi(z)
=
\frac{1}{k}
\sum_{i=1}^{k} z_i.
\]

For any vectors \(u,v\in\mathbb{R}^k\),

\[
\begin{aligned}
|\Phi(u)-\Phi(v)|
&=
\left|
\frac{1}{k}
\sum_{i=1}^{k}(u_i-v_i)
\right|
\\
&\le
\frac{1}{k}
\sum_{i=1}^{k}|u_i-v_i|
\\
&=
\frac{1}{k}\|u-v\|_1.
\end{aligned}
\]

Hence, the arithmetic mean is \(1/k\)-Lipschitz with respect to the
\(\ell_1\) norm, i.e.,

\[
L_\Phi=\frac{1}{k}.
\]

Substituting this Lipschitz constant into Theorem~6 yields

\[
|A_i(X)-A_i(X')|
\le
\frac{M}{k}|R_i|,
\]

which is precisely the desired bound.

\end{proof}
\begin{remark}

The proof of Theorem~6 decomposes the robustness of a local fairness
audit into two complementary components. Lemma~\ref{lem:lipschitz}
shows that the sensitivity of the audit is governed by the Lipschitz
continuity of the aggregation operator, independent of the source of
perturbation. Lemma~\ref{lem:replacement} quantifies the geometric
effect of perturbations by relating the difference between the
fairness-score vectors to the neighborhood replacement set $R_i$.

Combined with Theorem~\ref{thm:neighborhood}, Theorem~6 establishes the
robustness pipeline underlying our framework: bounded feature-space
perturbations induce bounded neighborhood replacement, and the resulting
audit deviation is bounded by the Lipschitz constant of the aggregation
operator and the extent of neighborhood replacement. Consequently,
aggregation operators with smaller Lipschitz constants exhibit greater
robustness to neighborhood perturbations. The arithmetic mean follows as
an immediate special case, while the same analysis extends to weighted
means, trimmed means, and other Lipschitz aggregation operators.

\end{remark}
\subsection{Appendix A5: Audit Volatility}
\label{appendix:A5}

The deterministic analysis in the previous subsection characterizes the
maximum deviation of a local fairness audit under a single perturbation.
In practice, however, robustness is often assessed over repeated
perturbations drawn from a stochastic perturbation model.
Theorem~\ref{thm:volatility} bounds the resulting audit volatility in
terms of the expected amount of neighborhood replacement.

\subsection{Proof of Theorem~\ref{thm:volatility}}

\begin{proof}

By Definition~\ref{def:volatility},

\[
\kappa_i
=
\mathbb E_{X'\sim\mathcal P}
\left[
|A_i(X)-A_i(X')|
\right].
\]

Applying the deterministic audit stability bound of
Theorem~\ref{thm:replacement} inside the expectation gives

\[
\begin{aligned}
\kappa_i
&=
\mathbb E
\left[
|A_i(X)-A_i(X')|
\right]
\\
&\le
\mathbb E
\left[
L_\Phi M |R_i|
\right].
\end{aligned}
\]

Since \(L_\Phi\) and \(M\) are constants independent of the
perturbation,

\[
\kappa_i
\le
L_\Phi M
\,
\mathbb E
\left[
|R_i|
\right].
\]

For arithmetic mean aggregation,
Corollary~\ref{cor:mean} implies
\(L_\Phi=1/k\). Hence,

\[
\kappa_i
\le
\frac{M}{k}
\,
\mathbb E
\left[
|R_i|
\right],
\]

which completes the proof.

\end{proof}

\begin{remark}

Theorem~\ref{thm:volatility} identifies two complementary factors that
govern audit volatility: the sensitivity of the aggregation operator,
captured by its Lipschitz constant \(L_\Phi\), and the expected amount
of neighborhood replacement induced by the perturbation model.
Consequently, audit volatility decreases when either the aggregation
operator is less sensitive (i.e., has a smaller Lipschitz constant) or
the perturbation model induces less neighborhood replacement.

Together, Corollary~\ref{cor:probabilistic},
Theorem~\ref{thm:replacement}, and
Theorem~\ref{thm:volatility} complete the robustness framework
developed in this paper. The local separation margin governs
probabilistic neighborhood preservation, neighborhood replacement
bounds the deviation of a single local fairness audit, and the expected
amount of neighborhood replacement determines audit volatility under
stochastic perturbations.

\end{remark}

\section{Appendix B: Additional Experimental Results}

To complement the experimental evaluation in Section~5, we provide an
additional empirical validation of Theorem~\ref{thm:volatility} at the
largest perturbation radius considered in our experiments
($\epsilon=0.20$). For arithmetic mean aggregation
($k=10$, $M=1$), Corollary~\ref{cor:mean} gives the theoretical bound

\[
\kappa_i
\le
\frac{M}{k}\,
\mathbb E\!\left[|R_i|\right].
\]

\begin{table}[t]
\centering
\caption{Empirical validation of
Theorem~\ref{thm:volatility}
at the largest perturbation radius
($\epsilon=0.20$). Across all datasets, the observed audit volatility
remains below the theoretical upper bound, with no per-individual
violations.}
\label{tab:theorem8-validation}
\small
\begin{tabular}{lrrrrrr}
\toprule
Dataset &
NPR &
$\mathbb E[|R_i|]$ &
$\kappa$ &
Bound &
$\kappa/\mathrm{Bound}$ &
Viol. \\
\midrule
Adult   & 0.765 & 0.520 & 0.00238 & 0.052 & 0.046 & 0 \\
Bank    & 0.705 & 0.634 & 0.001504 & 0.063 & 0.024 & 0 \\
COMPAS  & 0.515 & 2.196 & 0.004333 & 0.220 & 0.020 & 0 \\
\bottomrule
\end{tabular}
\end{table}

Across all three datasets, the observed audit volatility remains
substantially below the corresponding theoretical upper bound, and no
per-individual violations are observed. Although the theoretical bound
is conservative, it correctly captures the dependence of audit
volatility on the expected amount of neighborhood replacement. These
results provide empirical support for
Theorem~\ref{thm:volatility}, confirming that the proposed geometric
framework yields informative and reliable robustness guarantees in
practice.

\end{document}